\documentclass[journal,11pt,onecolumn]{IEEEtran}
\usepackage[ruled,vlined]{algorithm2e}
\usepackage{array}
\usepackage[caption=false,font=normalsize,labelfont=sf,textfont=sf]{subfig}
\usepackage{textcomp}
\usepackage{stfloats}
\usepackage{url}
\usepackage{verbatim}
\usepackage{graphicx}

\usepackage{color}
\usepackage[hidelinks]{hyperref}

\usepackage[numbers]{natbib}

\usepackage{balance}
\usepackage[dvipsnames]{xcolor}
\usepackage[many]{tcolorbox}
\usepackage{booktabs}

\usepackage{amsmath,amsfonts}
\usepackage{amssymb,amsthm}
\newtheorem{theorem}{Theorem}
\newtheorem{lemma}[theorem]{Lemma}

\newtheorem{proposition}[theorem]{Proposition}

\theoremstyle{definition}
\newtheorem{definition}{Definition}
\newtheorem{remark}{Remark}

\newtheorem{example}{Example}

\def\de{\overset{\Delta}{=}}

\SetKwInput{KwInput}{Input}
\SetKwInput{KwOutput}{Output}
\SetKwComment{Comment}{/*}{ */}

\title{Quickest Change Detection for Unnormalized Statistical  Models\footnotetext{A version of this paper has been accepted by the 26th International Conference on Artificial Intelligence and Statistics (AISTATS 2023).} }

\author{Suya Wu, Enmao Diao, Taposh Banerjee, Jie Ding, and Vahid Tarokh
\thanks{
Suya Wu, Enmao Diao, and Vahid Tarokh are with the Department of Electrical and Computer Engineering, Duke University, Durham, NC 27708 USA. (email: suya.wu@duke.edu; enmao.diao@duke.edu; vahid.tarokh@duke.edu) \\
\indent Taposh Banerjee is with the Department of Industrial Engineering, University of Pittsburgh, Pittsburgh, PA 15213 USA. (email: taposh.banerjee@pitt.edu)\\
\indent Jie Ding is with School of Statistics, University of Minnesota Twin Cities, Minneapolis, 
MN 55455 USA. (email: dingj@umn.edu)\\
\indent {Suya Wu and Vahid Tarokh were supported in part by Air Force Research Lab Award under grant number FA-8750-20-2-0504. Jie Ding was supported in part by the Office of Naval Research under grant number N00014-21-1-2590. Taposh Banerjee was supported in part by the U.S. Army Research Lab under grant W911NF2120295.}}}

\begin{document}
 \maketitle

\begin{abstract}
Classical quickest change detection algorithms require modeling pre-change and post-change distributions. Such an approach may not be feasible for various machine learning models because of the complexity of computing the explicit distributions. Additionally, these methods may suffer from a lack of robustness to model mismatch and noise. This paper develops a new variant of the classical Cumulative Sum (CUSUM) algorithm for the quickest change detection. This variant is based on Fisher divergence and the Hyv\"arinen score and is called the Score-based CUSUM (SCUSUM) algorithm. The SCUSUM algorithm allows the applications of change detection for unnormalized statistical models, i.e., models for which the probability density function contains an unknown normalization constant. The asymptotic optimality of the proposed algorithm is investigated by deriving expressions for average detection delay and the mean running time to a false alarm. Numerical results are provided to demonstrate the performance of the proposed algorithm.
\end{abstract}

\begin{IEEEkeywords}
Quickest change detection, CUSUM, Fisher divergence, Score matching, Unnormalized models
\end{IEEEkeywords}

\section{Introduction}
\IEEEPARstart{D}{etecting} abrupt changes in the underlying statistical characteristics of online data streams is an important problem commonly encountered in many applications. For example, this problem has applications in sensor networks, cyber-physical systems, biology, and neuroscience \cite{veeravalli2014quickest}.  
In the statistical problem of quickest change detection, observations are modeled as a realization of a stochastic process. The problem is posed as the problem of detecting a change in the distribution of a sequence of random variables. 
A change point is defined as a time when such a change in distribution occurs.
The quickest change detection algorithm can detect the change point as quickly as possible, with the minimum possible delay, subject to a constraint on the rate of false alarms~\cite{veeravalli2014quickest}. 
A typical quickest change detection algorithm is a single-threshold test where a sequence of statistics is computed over time, and an alarm is raised the first time the sequence is above a pre-designed threshold. The threshold is used to control the rate of false alarms. 

In the quickest change detection literature, the most well-studied setting is the independent and identically distributed (i.i.d.) setting. In this setting, it is assumed that the random variables are i.i.d. with a particular probability density function (written in short as density when there is no ambiguity) before the change, and are i.i.d. with another density after the change. In the i.i.d. setting, the main optimality results are obtained in \cite{shiryaev1963optimum},  \cite{lorden1971procedures}, \cite{pollak1985optimal}, \cite{moustakides1986optimal}. In \cite{shiryaev1963optimum}, it is shown that if the change point is modeled as a geometrically distributed random variable, then the optimal algorithm is to stop the first time the \textit{a posterior} probability that the change has already occurred is above a fixed threshold. This algorithm is also called the Shiryaev algorithm and is shown to minimize the average detection delay subject to a constraint on the probability of a false alarm.  In \cite{lorden1971procedures}, a novel minimax problem formulation is introduced, and it is shown that the Cumulative Sum (CUSUM) algorithm, proposed in  \cite{page1955test}, is asymptotically optimal, as the mean running time to a false alarm goes to infinity. In \cite{pollak1985optimal}, another variant of a minimax problem formulation is considered and it is shown that the Shiryaev-Roberts algorithm, proposed in \cite{roberts1966comparison}, is asymptotically optimal, as the mean running time to a false alarm goes to infinity. In \cite{moustakides1986optimal}, it is shown that the CUSUM algorithm is exactly optimal for the formulation in \cite{lorden1971procedures}. In \cite{lai1998information}, it is shown that the CUSUM algorithm is also asymptotically optimal with respect to the minimax variant studied in \cite{pollak1985optimal}. The classical i.i.d. setting has been extended to non-i.i.d. settings in \cite{lai1998information} and \cite{tartakovsky2005general}. For a more detailed discussion of the state-of-the-art theoretical results in this classical setting, we refer the reader to~\cite{veeravalli2014quickest, polunchenko2012state} and the references therein. 

One common feature of all the optimal algorithms in the quickest change detection literature is that the knowledge of the pre- and post-change densities are used to calculate the likelihood ratios of the observations. These likelihood ratios are then used to calculate the optimal change detection statistic 
\cite{shiryaev1963optimum},  \cite{lorden1971procedures}, \cite{pollak1985optimal}, \cite{moustakides1986optimal} \cite{lai1998information}, \cite{tartakovsky2005general}. In some machine learning applications, calculating the likelihood ratios can be computationally challenging. 

In some machine learning applications, the data models may be high-dimensional and, in some cases, may not lend themselves to explicit distributions. For example, energy-based models~\cite{LeCun2006ATO} capture dependencies between observed and latent variables based on their associated energy (an unnormalized probability), and score-based deep generative models~\cite{song2020score} generate high-quality images by learning the score function (the gradient of the log density function). These models can be computationally cumbersome to normalize themselves as probabilistic density functions, and therefore likelihood-based change detection algorithms are computationally expensive in implementation. In Subsection~\ref{subsec:issues_llr_cusum}, we show this difficulty with two examples. When the full knowledge of pre- and post-change distributions is not available, the data-generating distributions must be modeled using the available data. In such scenarios, likelihood-based detection algorithms do not perform as well as expected. For instance, by numerical results, \citet{chen2015graph} showed issues with the performance of generalized likelihood ratio-based algorithms when the dimension of data increases. For image datasets, Nalisnick et al. in \cite{nalisnick2018deep} demonstrated the likelihood learned from flow-based deep generative models cannot distinguish distribution drifts from one dataset to another.

Motivated by these limitations of likelihood-based procedures,  we propose a novel score-based quickest change detection algorithm that can be applied to unnormalized models, i.e., statistical models for which the density contains an unknown normalization constant. Specifically, we use the Hyv\"arinen score \cite{hyvarinen2005estimation} to propose a score-based variant of the CUSUM algorithm. In the following, we refer to this variant by the SCUSUM algorithm. The Hyv\"arinen score is proposed by~\citet{hyvarinen2005estimation} for establishing an empirical estimation procedure for unnormalized models. This estimation procedure is also known as score matching. Recently, \citet{wu2022score} proposed a score-based test statistic as a surrogate of the log-likelihood ratio statistic for unnormalized models. Their experimental results demonstrate significant performance gains and a reduction in computational complexity in testing unnormalized distribution drifts. In this paper, we consider the sequential version of the problem considered in \cite{wu2022score}. 

The delay and false alarm analysis of the CUSUM algorithm are performed using martingale and renewal theoretic methods \cite{lai1998information}, \cite{woodroofe1982nonlinear}. The martingale methods in particular utilize the fact that the likelihood ratios used by the CUSUM algorithm form a martingale under the pre-change model. The SCUSUM algorithm is a score-based algorithm, and the cumulative scores do not enjoy such a martingale characterization. To this end, in this paper, we employ novel analysis techniques to analyze the SCUSUM algorithm. 
We summarize the main contributions of this paper below.
\begin{enumerate}
    \item We propose the SCUSUM algorithm, a new quickest change detection algorithm that applies to unnormalized models for pre- and post-change distributions. In this algorithm, we replace the negative log-likelihood terms in CUSUM with a multiple of the Hyv\"arinen score. 
    \item We provide delay and false alarm analysis of the SCUSUM algorithm in the i.i.d. setting. We consider the minimax formulations of Lorden and Pollak~\cite{lorden1971procedures}, \cite{pollak1985optimal}. We prove that under the no-change assumption, the average run length (ARL) of SCUSUM (namely the mean running time it takes to declare a spurious change) increases exponentially as a function of the stopping threshold (Theorem~\ref{thm:arl}). Moreover, if a change occurs, we prove that the worst-case detection delay is a linear function of the stopping threshold (Theorem~\ref{thm:cond_edd}).
    \item We conduct extensive numerical experiments on synthetic data to demonstrate the performance of SCUSUM and compare it against likelihood ratio-based CUSUM~\cite{page1955test}, Scan B-statistic~\cite{li2019scan}, and CALM-MMD~\cite{cobb2022sequential}. Under the same constraint of ARL, our method performs competitively with CUSUM regarding empirical detection delay. In scenarios with non-Normal distributed data, SCUSUM outperforms Scan B-statistic and CALM-MMD. Our experiments further illustrate the computational advantage of SCUSUM over CUSUM for unnormalized models.
\end{enumerate}

The outline of our paper is as follows. In Sections~\ref{sec: background}, we formulate the quickest change detection problem and review the CUSUM algorithm. We also discuss in detail the limitations of the likelihood ratio-based procedures. In Section~\ref{sec: SCUSUM}, we discuss the concept of proper scoring rules and propose the SCUSUM algorithm. In Section~\ref{sec:theoritical_analysis}, we provide the delay and false alarm analysis of the SCUSUM algorithm. In Section~\ref{sec: results}, we present numerical experiments and compare the developed method with baseline methods. Finally, we conclude this work in Section~\ref{sec: conclusion}.

\section{Classical Quickest Change Detection}
\label{sec: background}
\subsection{Problem Formulation}
\label{subsec:problem_formulation}
\noindent Let $\{X_n\}_{n\geq 1}$ denote a sequence of independent random variables defined on the probability space $(\Omega, \mathcal{F}, P_\nu)$. Let $\mathcal{F}_n$ be the $\sigma-$algebra generated by random variables $X_1,\; X_2, \;\dots,\; X_n$ and $\mathcal{F}=\sigma(\cup_{n\geq 1}\mathcal{F}_n)$, the $\sigma-$algebra generated by the union of sub-$\sigma$-algebras. 
Under $P_\nu$, $X_1, \; X_2, \;\dots,\; X_{\nu-1}$ are {i.i.d.} according to a density $p_\infty$ and $X_{\nu}, \; X_{\nu+1},\; \dots$ are {i.i.d.} according to a density $p_1$. We think of $\nu$ as the change point, $p_\infty$ as the pre-change density, and $p_1$ as the post-change density. We use $\mathbb{E}_{\nu}$ and $\text{Var}_{\nu}$ to denote the expectation and the variance associated with the
measure $P_\nu$, respectively. We use $P_\infty$ to denote the measure under which there is no change, with $\mathbb{E}_\infty$ denoting the corresponding expectation.
A change detection algorithm is a stopping time $T$ with respect to the data stream $\{X_n\}_{n\geq 1}$:
$$
\{T \leq n\} \in \mathcal{F}_n.
$$
If $T\geq \nu$, we have made a \textit{delayed detection}; otherwise, a \textit{false alarm} has happened. 
Intuitively, there is a trade-off between detection delay and false alarms. 
We consider two minimax problem formulations to find the best stopping rule. 

In \cite{lorden1971procedures}, the following minimax metric, the worst-case averaged detection delay (WADD), is defined:
\begin{equation}
\label{eq: wadd}
\mathcal{L}_{\texttt{WADD}}(T)\de \sup_{\nu\geq 1}\text{ess}\sup \mathbb{E}_{\nu}[(T-\nu+1)^{+}|\mathcal{F}_{\nu}],
\end{equation}
where $(y)^{+}\de\max(y, 0)$ for any $y\in \mathbb{R}$. This leads to the minimax optimization problem 
\begin{equation}
    \label{eq:lorden}
    \min_T \;\mathcal{L}_{\texttt{WADD}}(T)\;
    \text{subject to}\;\mathbb{E}_{\infty}[T]\geq \gamma.
\end{equation}
We are also interested in the version of minimax metric introduced in \citet{pollak1985optimal}, the worst conditional averaged detection delay (CADD):
\begin{equation}
\label{eq:cadd}
    \mathcal{L}_{\texttt{CADD}}(T)\de \sup_{\nu\geq 1}\mathbb{E}_{\nu}[T-\nu|T\geq \nu].
\end{equation}
The optimization problem becomes
\begin{equation}
    \label{eq:pollak}
    \min_T \;\mathcal{L}_{\texttt{CADD}}(T)\;
    \text{subject to}\;\mathbb{E}_{\infty}[T]\geq \gamma. 
\end{equation}

\subsection{The Likelihood Ratio-based CUSUM Algorithm}
\label{subsec:llr-cusum}
\noindent Given the data stream $\{X_n\}_{n\geq 1}$, the stopping rule of the likelihood ratio-based CUSUM algorithm is defined by
\begin{equation*}
\label{eq:cusumrule}
    T_{\texttt{CUSUM}} \de \inf \biggl\{n\geq 1: \max_{1\leq k\leq n}\sum_{i=k}^n\log \frac{p_{1}(X_i)}{p_{\infty}(X_i)}\geq \tau \biggr\},
\end{equation*}
where the infimum of the empty set is defined to be $+\infty$, and $\tau>0$ is referred to as the stopping threshold. The value of this threshold is clearly related to the trade-off between detection delay and false alarms. It is known~\cite{lai1998information} that $T_{\texttt{CUSUM}}$ can be written as
\begin{equation*} 
    T_{\texttt{CUSUM}}=\inf\{n\geq 1:\Lambda(n)\geq \tau\},
\end{equation*}
where $\Lambda(n)$ is defined using the recursion
\begin{align}
    &\Lambda(0)=0, \nonumber \\
    &\Lambda(n) \de \biggr(\Lambda(n-1)+\log \frac{p_1(X_n)}{p_{\infty}(X_n)}\biggr)^{+}, \forall n \geq 1, \label{eq:cusum_score}
 \end{align}
which leads to a computationally convenient stopping scheme. 

In \cite{moustakides1986optimal}, it is shown that the CUSUM algorithm is exactly optimal, for every fixed constraint $\gamma$, for Lorden's problem 
\eqref{eq: wadd}. As pointed in \cite{lai1998information}, the algorithm is also asymptotically optimal for Pollak's problem \eqref{eq:cadd}. In \cite{lorden1971procedures} and \cite{lai1998information}, the asymptotic performance of the CUSUM algorithm is also characterized. Specifically, it is shown that
\begin{align}
\label{eq:optimality_cusum}
    \mathcal{L_{\texttt{WADD}}}(T_{\texttt{CUSUM}}) \sim \mathcal{L_{\texttt{CADD}}}(T_{\texttt{CUSUM}})\sim \frac{\log \gamma}{\mathbb{D}_{\texttt{KL}}(P_{1}\|P_{\infty})},\; \text{as}\; \gamma \rightarrow \infty.
\end{align}
Here $\mathbb{D}_{\texttt{KL}}(P_{1}\|P_{\infty})$ is the Kullback-Leibler divergence between the post-change distribution $P_{1}$ (associated with the density $p_1$) and pre-change distribution $P_{\infty}$ (associated with density $p_\infty$):
$$
\mathbb{D}_{\texttt{KL}}(P_{1}\|P_{\infty}) = \int_x p_1(x) \log \frac{p_1(x)}{p_\infty(x)} dx, 
$$
and the notation $g(c)\sim h(c)$ as $c\to c_0$ indicates that $\frac{g(c)}{h(c)} \to 1$ as $c\to c_0$ for any two functions $c\mapsto g(c)$ and $c\mapsto h(c)$.
\subsection{Issues with the Likelihood Ratio-based CUSUM Algorithm}
\label{subsec:issues_llr_cusum}
\noindent We consider the pre- and post-change densities, $p_1$ and $p_{\infty}$, respectively. We assume that the densities are potentially known only up to a normalizing constant, i.e., we have unnormalized models. In other words, instead of $p_1(x)$ and $p_{\infty}(x)$, we are given $\tilde{p}_1(x)$ and $\tilde{p}_{\infty}(x)$ with
\begin{equation*}
    p_{i}(x) = \frac{\tilde{p}_{i}(x)}{\int_{x\in \mathcal{X}} \tilde{p}_i(x)dx}, \; i=1,\infty.
\end{equation*}
As discussed in the introduction, such models are occasionally encountered in several machine learning applications. 
In many cases, the computation of the denominator (also known as the \textit{normalizing constant} or the \textit{partition function}) can be intractable when the integral is not analytic in a closed form. For low-dimensional cases, numerical integration can be used to approximate the function. However, the number of points required for approximating the integral may grow exponentially as a function of the dimension of data space. This approximation is computationally expensive for high-dimensional data. Hence, implementing the likelihood ratio-based CUSUM algorithm is computationally cumbersome for unnormalized models. Next, we provide two examples to show this issue.
\begin{example}[Exponential Family] We consider a subfamily of the Exponential family belonging to pairwise interaction graphical models~\citep{yu2016statistical}. Let $X\in \mathbb{R}^d$ be the random variable, and let $p_{\tau}$ represent the density, which is formulated as
\begin{align}
    p_{\tau}(X) =\frac{1}{Z_{\tau}} \exp\left\{-\tau\left(\sum_{i=1}^dx_i^4+\sum_{1\leq i\leq d, i\leq j\leq d}x_i^2x_j^2\right)\right\},\nonumber
\end{align}
where $\tau\in \mathcal{T}\subset\mathbb{R}^{+}$ is the model parameter and $Z_{\tau}$ is the normalizing constant of $p_{\tau}(X)$. Here,\begin{equation*}
    Z_{\tau} = \int_{x_1}\cdots \int_{x_d}\exp\left\{-\tau\left(\sum_{i=1}^dx_i^4+\sum_{1\leq i\leq d, i\leq j\leq d}x_i^2x_j^2\right)\right\}dx_1\cdots d x_d.
\end{equation*}
As shown above, this integral cannot be computed in a closed form, and therefore the density $p_{\tau}$ cannot be computed in a closed form. Besides, the numerical approximation is time-consuming when $d$ is large. Particularly, in Section~\ref{sec: results}, we show that the likelihood ratio-based CUSUM cannot be implemented in a reasonable computational time when $d=4$.
\end{example}
\begin{example}[Gauss-Bernoulli Restricted Boltzmann Machine]
Restricted Boltzmann Machine (RBM)~\citep{LeCun2006ATO} is a generative graphical model defined on a bi-partite graph of hidden and visible variables. In particular, we consider the Gauss-Bernoulli RBM (GB-RBM), which has binary-valued hidden variables $H=(h_1, \ldots, h_{d_h})^{T}\in \{0,1\}^{d_h}$, real-valued visible variables $X=(x_1, \ldots, x_{d_x})^{T}\in R^{d_x}$, and the joint density 
\begin{equation*}
    p(X, H) = \frac{1}{Z}\text{exp} \left\{
    -\left(\frac{1}{2}\sum_{i=1}^{d_x}\sum_{j=1}^{d_h}\frac{x_i}{\sigma_i}W_{ij}h_j\right.\left.+\sum_{i=1}^{d_x}b_ix_i+\sum_{j=1}^{d_h}c_jh_j-\frac{1}{2}\sum_{i=1}^{d_{x}} \frac{x_{i}^{2}}{\sigma_{i}^{2}}\right)
    \right\}, 
\end{equation*}
where model parameters $\theta = (\mathbf{W}, \mathbf{b}, \mathbf{c})$ and $Z$ is the normalizing constant of $p(X, H)$. We set $\sigma_i=1$ for all $i=1,\dots, d_x$. 

Let $p_{\theta}$ represent the density of the visible variable $X$, which can be written as $$
p_{\theta}(X)= \sum_{h\in \{0,1\}^{d_h}}p_{\theta}(X, H) = \frac{1}{Z_{\theta}}\exp\{-F_{\theta}(X)\},
$$
where $Z_{\theta}$ is the normalizing constant of $p_{\theta}(X)$, and $F_{\theta}(X)$ is the free energy given by 
\begin{equation*}
    F_{\theta}(X) = \frac{1}{2}\sum_{i=1}^{d_x} (x_{i}-b_i)^{2}\nonumber
    -\sum_{j=1}^{d_h} \operatorname{Softplus}\left(\sum_{i=1}^{d_x} W_{i j}x_{i}+b_{j}\right).
\end{equation*}
The $\operatorname{Softplus}$ function is defined as $\operatorname{Softplus}(y) \de \log(1+\exp(y))$ with a default scale parameter $\beta=1$. The same computational difficulty occurs on $p_{\theta}(X)$, and therefore the likelihood of the GB-RBM data may not be computed exactly in practice. 
\end{example}
\section{Score-based Quickest Change Detection}
\label{sec: SCUSUM}
\noindent In this section, we propose a score-based CUSUM (SCUSUM) algorithm to address the issues with likelihood ratio-based CUSUM for unnormalized models. Following the scheme of CUSUM, the proposed algorithm can be applied in a recursive way, which is not too demanding in computational and memory requirements for online implementation. To this end, we first review the framework of proper scoring rules, which build an intuitive comparison between CUSUM and SCUSUM.
\subsection{Proper Scoring Rules}
\label{subsec:score_rules}
\noindent Let $X$ be a random variable with values in $\mathcal{X}\subseteq \mathbb{R}^d$, and let $\mathcal{P}$ be a family of distributions over $\mathcal{X}$. Let $P$ and $Q \in \mathcal{P}$ denote the true data-generating distribution and a postulated distribution, and let $p$ and $q$ respectively denote their corresponding densities. \citet{gneiting2007strictly} studied proper scoring rules as a unified framework to measure the quality of postulated models on observed data.

\begin{definition}[Proper Scoring Rule] A scoring rule is a function $(X, Q)\mapsto \mathcal{S}(X,Q)$ that measures the quality of $Q$ for modeling data represented by $X$. It is said to be \textit{proper} if for all $P \in \mathcal{P}$, the expected score $\mathbb{E}_{X\sim P}\mathcal{S}[(X, Q)]$ is minimized at $Q=P$, where the minimum is taken over all $Q \in \mathcal{P}$. 
Moreover, $\mathcal{S}$ is \textit{strictly proper} with respect to $\mathcal{P}$, if for any $Q\in \mathcal{P}$ and $Q\neq P$, $\mathbb{E}_{X\sim P}[\mathcal{S}(X, Q)] > \mathbb{E}_{X\sim P}[\mathcal{S}(X, P)]$. 
\end{definition}
The logarithmic scoring rule~\cite{good1992rational} is a well-known and widely applied example of a strictly proper scoring rule. 
\begin{definition}[Logarithmic Score]
    The logarithmic scoring rule (also called the log score) is given by
    \begin{equation*}
        (X, Q)\mapsto \mathcal{S}_{\texttt{L}}(X, Q) \de -\log q(X).
    \end{equation*}
\end{definition}
Minimizing the log score is associated with maximum likelihood estimation (MLE) and the Kullback-Leibler (KL) divergence
\begin{equation*}
    \mathbb{D}_{\texttt{KL}}(P\|Q)\de \mathbb{E}_{X\sim P} \left[\log p(X) - \log q(X)\right].
\end{equation*} 
Since $\mathbb{D}_{\texttt{KL}}(P\|Q) > 0$ for any $Q\neq P$, the log score is \textit{strictly proper}. The detection score of LLR-based CUSUM, defined in Equation~(\ref{eq:cusum_score}), can be rewritten by $$
\log \frac{p_1(X_n)}{p_{\infty}(X_n)} = \mathcal{S}_{\texttt{L}}(X_n, P_{\infty})-\mathcal{S}_{\texttt{L}}(X_n, P_1).$$
\subsection{Fisher divergence and Hyv\"arinen score}
\label{subsec:fisher_hyvarinen}
\noindent \citet{hyvarinen2005estimation} proposed an estimation procedure for unnormalized statistical models by minimizing the Fisher divergence from $P$ to $Q$, defined by
\begin{align*}
    \mathbb{D}_{\texttt{F}} (P \| Q) \de \mathbb{E}_{X\sim P} \left[\left \| \nabla_{\mathbf{x}} \log p(X)- \nabla_{\mathbf{x}} \log q(X)\right \|_2^2 \right],
\end{align*}
where $\|\cdot\|_2$ denotes the Euclidean norm. Clearly, $\nabla_{\mathbf{x}} \log p(X)$ and $\nabla_{\mathbf{x}} \log q(X)$ remain invariant if $p$ and $q$ are scaled by any positive constant with respect to $X$. Hence, the Fisher divergence remains \textit{scale-variant} with respect to an arbitrary constant scaling of density functions. Under some mild regularity conditions on $p$ and $q$, \citet{hyvarinen2005estimation} showed that
\begin{align*}
    \mathbb{D}_{\texttt{F}} (P \| Q) =\mathbb{E}_{X\sim P} \left[\frac{1}{2}\left \| \nabla_{\mathbf{x}} \log p(X) \right \|_2^2 + \mathcal{S}_{\texttt{H}}(X, Q)\right],
\end{align*}
where $\mathcal{S}_{\texttt{H}}(X, Q)$ a \textit{scale-invariant} proper scoring function, referred to as the Hyv\"arinen score in the framework of proper scoring rules~\cite{parry2012proper}. Since $\frac{1}{2}\left \| \nabla_{\mathbf{x}} \log p(X) \right \|_2^2$ is a constant in terms of $Q$, then minimizing the Fisher divergence is equivalent to minimizing $\mathcal{S}_{\texttt{H}}(X, Q)$.
\begin{definition}[Hyv\"arinen Score] The Hyv\"arinen score is a mapping  $(X, Q)\mapsto \mathcal{S}_{\texttt{H}}(X, Q)$ given by 
    \begin{equation}
        \label{eq:hyv_score}
        \mathcal{S}_{\texttt{H}}(X, Q) \de \frac{1}{2} \left \| \nabla_{X} \log q(X) \right \|_2^2 + \Delta_{X} \log q(X)
    \end{equation}
whenever it can be well defined. Here, $\nabla_{X}$ and $\Delta_{X} = \sum_{i=1}^d \frac{\partial^2}{\partial x_i^2}$ respectively denote the gradient and the Laplacian operators acting on $X = (x_1, \cdots, x_d)^{\top}$.
\end{definition}
$\mathcal{S}_{\texttt{H}}$ is \textit{scale-invariant} inherited from the \textit{scale-invariant} property of Fisher divergence. This property avoids the computation of the normalizing constant for unnormalized models. Specifically, when the knowledge of $Q$ is up to $\tilde{q}(x)$ with
\begin{equation*}
    q(x) = \frac{\tilde{q}(x)}{\int_{x\in \mathcal{X}} \tilde{q}(x)dx},
\end{equation*} 
it is easy to see that $\mathcal{S}_{\texttt{H}}(X, Q)$ remains invariant by replacing the density $q$ with the associated unnormalized term $\tilde{q}$. Additionally, it is easy to verify that $\mathbb{D}_{\texttt{F}}(P\| Q) > 0$ for $Q \neq P$, thus the Hyv\"arinen score is \textit{strictly proper}.

\subsection{The Score-based CUSUM Algorithm}
\label{subsec: SCUSUM_algorithm}
\noindent From the discussion in Subsections \ref{subsec:score_rules} and \ref{subsec:fisher_hyvarinen}, the Hyv\"arinen score function can be seen as a surrogate of the log score function. Motivated by this analogy, we consider replacing the log scores with the Hyv\"arinen scores in the LLR-based CUSUM algorithm. Next, we define the detection score of SCUSUM and then provide the stopping scheme. 

Let $X$ represent a generic random variable defined on the probability space $(\Omega, \mathcal{F}, P)$. $P$ could be either the pre- or post-change distribution. We define the instantaneous SCUSUM score function $X\mapsto z_{\lambda}(X)$ by 
\begin{equation}
\label{eq:scusum_instantaneous}
    z_{\lambda}(X) \de \lambda\bigr(\mathcal{S}_{\texttt{H}}(X, P_{\infty})-\mathcal{S}_{\texttt{H}}(X, P_{1})\bigr),
\end{equation}
where $\lambda>0$ is a pre-selected multiplier, $\mathcal{S}_{\texttt{H}}(X, P_{\infty})$ and $\mathcal{S}_{\texttt{H}}(X, P_1)$ are respectively the Hyv\"arinen score functions of pre- and post-change distributions. In Section~\ref{sec:theoritical_analysis}, we will provide a detailed discussion on the role of $\lambda$ in the SCUSUM algorithm. Then, our proposed stopping rule is given by 
\begin{equation}
\label{eq:SCUSUM_rule}
    T_{\texttt{SCUSUM}} \de \inf \biggl\{n\geq 1: \max_{1\leq k\leq n}\sum_{i=k}^nz_{\lambda}(X_i)\geq \tau \biggr\},
\end{equation}
where $\tau>0$ is a stopping threshold, which is usually pre-selected to control false alarms. Similar to the stopping scheme of CUSUM, the stopping rule of SCUSUM can be written as
\begin{equation*} 
    T_{\texttt{SCUSUM}}=\inf\{n\geq 1:Z(n)\geq \tau\},
\end{equation*}
where $Z(n)$ can be computed recursively by
\begin{align*}
    &Z(0)=0, \\
    &Z(n) \de (Z(n-1)+z_{\lambda}(X_n))^{+},\;\forall n\geq 1.
\end{align*}
$Z(n)$ is referred to as the detection score of SCUSUM at time $n$. The SCUSUM algorithm is summarized in Algorithm~\ref{algm:scusum}.


\begin{algorithm}
\DontPrintSemicolon
\caption{SCUSUM Detection Algorithm}
\label{algm:scusum}
\KwInput{Hyvarinen score functions $\mathcal{S}_{\texttt{H}}(\cdot, P_{\infty})$ and $\mathcal{S}_{\texttt{H}}(\cdot, P_{1})$ of pre- and post-change distributions, respectively.} 
\KwData{$m$ previous observations $\mathbf{X}_{[-m+1,0]}$ and the online data stream $\{X_n\}_{n\geq 1}$}
\SetKwProg{Fn}{Initialization}{:}{}
  \Fn{}{
       Current time $k=0$, $\lambda>0$, $\tau>0$, and $Z(0)=0$}
\While{$Z(k)<\tau$}{
$k = k+1$\;
Update $z_{\lambda}(X_k) = \lambda(\mathcal{S}_{\texttt{H}}(X_{k}, P_{\infty})-\mathcal{S}_{\texttt{H}}(X_{k}, P_{1}))$\;
Update $Z(k) = \max(Z(k-1)+z_{\lambda}(X_k), 0)$\;
}
Record the current time $k$ as the stopping time $T_{\texttt{SCUSUM}}$ \;
\KwOutput{$T_{\texttt{SCUSUM}}$}
\end{algorithm}

\section{Delay and False Alarm Analysis of the SCUSUM Algorithm}
\label{sec:theoritical_analysis}
\noindent In this section, we provide delay and false alarm analysis of the SCUSUM algorithm using the same notations and under the same problem setting defined in Section~\ref{sec: background} and Section~\ref{sec: SCUSUM}. We introduce two assumptions: 1) $P_{1}\neq P_{\infty}$, and 2) the same mild regularity conditions\footnote{We refer the details to \cite{hyvarinen2005estimation}.} made by \citet{hyvarinen2005estimation} so that the Hyv\"arinen score is well-defined. 

We first provide an overview of the results in this section. In Lemma~\ref{lemma: drifts}, we show that, just as in the CUSUM algorithm, the drift of the SCUSUM algorithm is negative before the change and positive after the change. The role played by the KL-divergence in the CUSUM algorithm is replaced by the Fisher divergence in the SCUSUM algorithm. Our core results are presented in Theorems~\ref{thm:arl} and~\ref{thm:cond_edd}. In Theorem~\ref{thm:arl}, we provide a lower bound of the average run length when no change has occurred. As discussed in the introduction, the challenge here is that the bound cannot be derived using classical martingale techniques, e.g. those employed in \cite{lai1998information}. 
This is because the SCUSUM algorithm is based on scores and not log-likelihood ratios. The latter have martingale properties that are employed by classical proofs. Our novel proof technique is developed after Lemma~\ref{lemma: drifts}, in Lemma~\ref{lemma: lambda}, and in the proof of Theorem~\ref{thm:arl}. 
In Theorem~\ref{thm:cond_edd}, we demonstrate an upper bound of the expected detection delay when a change point occurs at $v=1$. This, in turn, provides an upper bound on the $\mathcal{L}_{\texttt{WADD}}$ (see \eqref{eq: wadd}) of the SCUSUM algorithm. 
 In Proposition~\ref{prop: gaussian}, we consider a special case of multivariate Normal pre- and post-change distributions and discuss the asymptotic optimality of our algorithm in this particular case. Specifically, we show that in this case the KL-divergence and the Fisher divergence coincide and the SCUSUM algorithm has the same optimality properties as the CUSUM algorithm. 

\begin{lemma}[Positive and Negative Drifts]
\label{lemma: drifts}
Consider the instantaneous SCUSUM score function $X\mapsto z_{\lambda}(X)$ as defined in Equation~(\ref{eq:scusum_instantaneous}). Then,
\begin{align*}
\label{eq:expst}
    &\mathbb{E}_{\infty}\left[z_{\lambda}(X)\right] = -\lambda\mathbb{D}_{\texttt{F}}(P_{\infty} \| P_1)<0,\; \text{and}\\
    &\mathbb{E}_{1}\left[z_{\lambda}(X)\right] = \lambda\mathbb{D}_{\texttt{F}}(P_1 \| P_{\infty})>0.
\end{align*}
\end{lemma}
\begin{proof}
Under some mild regularity conditions, \citet{hyvarinen2005estimation} proved that
\begin{align*}
    \mathbb{D}_{\texttt{F}} (P \| Q) =\mathbb{E}_{X\sim P} \left[\frac{1}{2}\left \| \nabla_{X} \log p(X) \right \|_2^2 + \mathcal{S}_{\texttt{H}}( X, Q)\right].
\end{align*}
Let $C(P)$ denote $\mathbb{E}_{X\sim P} \left[\frac{1}{2}\left \| \nabla_{X} \log p(X) \right \|_2^2\right]$ for any $P\in \mathcal{P}$, then 
\begin{equation*}
     \mathbb{E}_{\infty}[\mathcal{S}_{\texttt{H}}(X, P_{\infty})-\mathcal{S}_{\texttt{H}}(X, P_1)]=\mathbb{D}_{\texttt{F}} (P_{\infty} \| P_{\infty})-C(P_{\infty})-\mathbb{D}_{\texttt{F}} (P_{\infty} \| P_1)+C(P_{\infty})=-\mathbb{D}_{\texttt{F}} (P_{\infty} \| P_1),
\end{equation*}
and 
\begin{equation*}
     \mathbb{E}_{1}[\mathcal{S}_{\texttt{H}}(X, P_{\infty})-\mathcal{S}_{\texttt{H}}(X, P_1)]=\mathbb{D}_{\texttt{F}} (P_1 \| P_{\infty})-C(P_{1})-\mathbb{D}_{\texttt{F}} (P_1 \| P_1)+C(P_{1})=\mathbb{D}_{\texttt{F}} (P_1 \| P_{\infty}).
\end{equation*}
Since $\lambda>0$ is a constant with respect to $P_1$ and $P_{\infty}$, the proof is complete.
\end{proof}
Lemma~\ref{lemma: drifts} shows that, prior to the change, the expected mean of instantaneous SCUSUM score $z_{\lambda}(X)$ is negative under the measurement of random observations. Consequently, the accumulated score has a negative drift at each time $n$ prior to the change. Thus, the SCUSUM detection score $Z(n)$ is pushed toward zero before the change point. This intuitively makes a false alarm unlikely. In contrast, after the change, the instantaneous score has a positive mean, and the accumulated score has a positive drift. Thus, the SCUSUM detection score will increase toward infinity and leads to a change detection event.

Next, we discuss the values of the multiplier $\lambda$ in the theoretical analysis. Obviously, with a fixed stopping threshold, a larger value of $\lambda$ results in a smaller detection delay because the increment of the SCUSUM detection score is large, and the threshold can be easily reached. However, a larger value of $\lambda$ also causes SCUSUM to stop prematurely when no change occurs, leading to a larger false alarm probability. Hence, except in some degenerate cases where the Hyv\"arinen score functions $P_{\infty}(S_{\texttt{H}}(X, P_1)-S_{\texttt{H}}(X, P_{\infty})\le 0)=1$, the value of $\lambda$ cannot be arbitrarily large. It needs to satisfy the following key condition:
\begin{equation}
\label{eq: condition}    
\mathbb{E}_{\infty}[\exp(z_{\lambda}(X))]\leq 1.
\end{equation}
We will present a technical lemma that guarantees the existence of such a $\lambda$ to satisfy Inequality~(\ref{eq: condition}). 

\begin{lemma}[Existence of appropriate $\lambda$]
    \label{lemma: lambda}
There exists $\lambda>0$ such that Inequality~(\ref{eq: condition}) holds. Moreover, either 1) there exists $\lambda^{\star} \in (0,\infty)$ such that the equality of~(\ref{eq: condition}) holds, or 2) for all $\lambda>0$, the inequality of~(\ref{eq: condition}) is strict.
\end{lemma}
\begin{proof}
Define the function $\lambda:\mapsto h(\lambda)$ given by $$h(\lambda)\de\mathbb{E}_{\infty}[\exp (z_{\lambda}(X))]-1.$$ Observe that \begin{equation*}
  h^{\prime}(\lambda)\de \frac{d h}{d\lambda}(\lambda)=\mathbb{E}_{\infty}[(S_{\texttt{H}}(X,P_{\infty})-S_{\texttt{H}}(X,P_{1}))\exp (z_{\lambda}(X))].
\end{equation*}
Note that $h(0)=0$, and $h^{\prime}(0)=-\mathbb{D}_{\texttt{F}}(P_{\infty}\|P_1)<0$. Thus, there exists $\lambda>0$ such that $h(\lambda)< 0$, and Inequality~(\ref{eq: condition}) is satisfied. 

Next, we prove that either 1) there exists $\lambda^{\star} \in (0,\infty)$ such that $h(\lambda^{\star}) = 0$, or 2) for all $\lambda>0$ we have $h(\lambda)<1$. 

Observe that
\begin{equation*}
    h''(\lambda)\de\frac{d^2 h}{d \lambda}(\lambda)
    =\mathbb{E}_{\infty}[(S_{\texttt{H}}(X,P_{\infty})-S_{\texttt{H}}(X,P_{1}))^2\exp (z_{\lambda}(X))]\geq 0.
\end{equation*}
We claim that $h(\lambda)$ is \textit{strictly convex}, namely $h''(\lambda) > 0$ for all $\lambda\in [0,\infty)$. Suppose $h''(\lambda) = 0$ for some $\lambda \geq 0$, we must have $S_{\texttt{H}}(X,P_{\infty})-S_{\texttt{H}}(X,P_{1}) = 0$ almost surely.  This implies that 
$\mathbb{E}_{\infty}[(S_{\texttt{H}}(X,P_{\infty})-S_{\texttt{H}}(X,P_{1}))]
= 0$ which in turn gives $-\mathbb{D}_{\texttt{F}}(P_{\infty}\|P_1) =0$ and $P_{\infty} = P_1$ almost everywhere, leading to a contradiction to the assumption $P_{\infty}\neq P_1$. Thus, $h(\lambda)$ is \textit{strictly convex} and $h^{\prime}(\lambda)$ is \textit{strictly increasing}. 

It follows that either 1) $h(\lambda)$ have at most one global minimum in $(0, \infty)$, or 2) it is strictly decreasing in $[0,\infty)$. We recognize two cases, and we show that the second case is degenerate that is of no practical interest.
\begin{itemize}
\item \textbf{Case 1:} If the global minimum of $h(\lambda)$ is attained at $a \in (0, \infty)$, then $h^{\prime}(a) = 0$. Since $h^{\prime}(0) < 0$ and $h(0) = 0$, the global minimum $h(a)<0$. 
Since $h^{\prime}(\lambda)$ is \textit{strictly increasing}, we can choose $b > a$ and conclude that $h^{\prime}(\lambda) > h^{\prime}(b) > h^{\prime}(a) = 0$ for all $\lambda > b$. It follows that $\lim_{\lambda \rightarrow \infty} h(\lambda) = +\infty$. Combining this with the continuity of $h(\lambda)$, we conclude that $h(\lambda^*) = 0$ for some $\lambda^* \in (0, \infty)$ and any value of $\lambda \in (0, \lambda^*]$ satisfies Inequality~(\ref{eq: condition}).

Note that in this case, we must have $P_{\infty}\left(S_{\texttt{H}}(X,P_{\infty})-S_{\texttt{H}}(X,P_{1}) \ge c\right)>0$, for some $c>0$. Otherwise, we have $P_{\infty}\left(S_{\texttt{H}}(X,P_{\infty})-S_{\texttt{H}}(X,P_{1}) \le 0\right)=1$. This implies that $P_{\infty}(z_{\lambda}(X)\le 0)=1$, or equivalently $\mathbb{E}_{\infty}[\exp (z_{\lambda}(X))]< 1$ for all $\lambda > 0$, and therefore leads to Case 2: $h(\lambda)< 0$ for all $\lambda > 0$. Here, $\mathbb{E}_{\infty}[\exp (z_{\lambda}(X))]\neq 1$ since $P_{\infty}(S_{\texttt{H}}(X,P_{\infty})-S_{\texttt{H}}(X,P_{1})=0)<1$; otherwise $P_{\infty}(S_{\texttt{H}}(X,P_{\infty})-S_{\texttt{H}}(X,P_{1})=0)=1$, and then $\mathbb{E}_{\infty}[S_{\texttt{H}}(X,P_{\infty})-S_{\texttt{H}}(X,P_{1})]=-\mathbb{D}_{\texttt{F}}(P_{\infty}\|P_{1})=0$, causing the same contradiction to $P_1\neq P_{\infty}$.

\item \textbf{Case 2:} If $h(\lambda)$ is strictly decreasing in $(0, \infty)$, then any $\lambda \in (0, \infty)$ satisfies Inequality~(\ref{eq: condition}). As discussed before, in this case, we must have $P_{\infty}\left(S_{\texttt{H}}(X, P_{\infty})-S_{\texttt{H}}(X, P_{1}) \le 0\right)=1$. Equivalently, all the increments of the SCUSUM detection score are non-positive under the pre-change distribution, and $P_{\infty}(Z(n)=0)=1$ for all $n$. Accordingly, $\mathbb{E}_{\infty}[T_{\textit{SCUSUM}}]=+\infty$. When there occurs change (under measure $P_{1}$), we also observe that SCUSUM can get close to detecting the change point instantaneously as $\lambda$ is chosen arbitrarily large. Obviously, this case is of no practical interest.
\end{itemize}

\end{proof}
From now on, we consider a fix $\lambda > 0$ that satisfies Inequality~(\ref{eq: condition}) to present our core results. In practice, it is possible to use $m$ past samples $\mathbf{X}_{[-m+1,0]}$ to determine the value of $\lambda$. In particular, $\lambda$ can be chosen as the positive root of the function $\lambda \to \tilde{h}(\lambda)$ given by 
\begin{align}
\label{eq:empirical_conditon}
\tilde{h}(\lambda)\de\frac{1}{m}\sum_{i=1}^m[\exp(z_{\lambda}(X_{i-m}))]-1.
\end{align}
By Lemma~\ref{lemma: lambda} and its related technical discussions, the above equation has a root greater than zero with a high probability if $m$ is sufficiently large. In the case that $\lambda$ is not chosen properly, the algorithm remains implementable but the performance of detection delay is not guaranteed. We discuss this situation further in Remark~\ref{remark: lambda}.

\begin{theorem}
\label{thm:arl}
Consider the stopping rule $T_{\texttt{SCUSUM}}$ defined in Equation~(\ref{eq:SCUSUM_rule}). Then, for any $\tau>0$,
    \begin{equation}
    \label{eq:arl}
        \mathbb{E}_{\infty}[T_{\texttt{SCUSUM}}]\geq  e^{\tau}.
    \end{equation}
\end{theorem}
\begin{proof}
We follow the proof of \citet[Theorem 4]{lai1998information} to conclude the result of Theorem~\ref{thm:arl}. A constructed martingale and Doob's submartingale inequality~\citep{doob1953stochastic} are combined to finish the proof. 
\begin{enumerate}
    \item We first construct a non-negative martingale with mean $1$ under the measure $P_{\infty}$. Define a new instantaneous score function $X \mapsto \tilde{z}_{\lambda}(X)$ given by 
\begin{equation*}
    \label{eq:new_instant_z_lambad}
    \tilde{z}_{\lambda}(X)\de z_{\lambda}(X)+\delta,
\end{equation*}
where $$\delta \de -\log \biggr(\mathbb{E}_{\infty}\left[\exp (z_{\lambda}(X))\right]\biggr).$$ Further define the sequence $$\tilde{G}_n\de \exp \biggr(\sum_{k=1}^n\tilde{z}_{\lambda}(X_k)\biggr),\; \forall n\geq 1.$$ 

Suppose $X_1, X_2, \ldots$ are i.i.d according to $P_{\infty}$ (no change occurs). Then,
\begin{align*}
    \mathbb{E}_{\infty}\left[\tilde{G}_{n+1}\mid \mathcal{F}_n\right] = \tilde{G}_n\mathbb{E}_{\infty}[\exp(\tilde{z}_{\lambda}(X_{n+1}))]=\tilde{G}_{n}e^{\delta}\mathbb{E}_{\infty}[\exp(z_{\lambda}(X_{n+1}))]=\tilde{G}_{n},
\end{align*}
and
\begin{align*}
    \mathbb{E}_{\infty}[\tilde{G}_n] &= \mathbb{E}_{\infty}\left[\exp\left(\sum_{i=1}^{n}(z_{\lambda}(X_i)+\delta)\right)\right]= e^{n\delta} \prod_{i=1}^n\mathbb{E}_{\infty}[\exp(z_{\lambda}(X_i))]=1.
\end{align*}
Thus, under the measure $P_{\infty}$, $\{\tilde{G}_n\}_{n\geq 1}$ is a non-negative martingale with the mean $\mathbb{E}_{\infty}[\tilde{G}_1]=1$. 

\item We next examine the new stopping rule 
\begin{equation*}
    \tilde{T}_{\texttt{SCUSUM}} = \inf \left\{n\geq 1: \max_{1\leq k\leq n} \sum_{i=k}^n \tilde{z}_{\lambda}(X_i)\geq \tau \right\},
\end{equation*}
where $\tilde{z}_{\lambda}(X_i) = z_{\lambda}(X_i)+\delta$. By Inequality~(\ref{eq: condition}), we observe that $\delta\geq 0$. By Jensen's inequality,
\begin{equation}
\label{eq:jensen}
    \mathbb{E}_{\infty}[\exp(z_{\lambda}(X))]\geq \exp\left(\mathbb{E}_{\infty}[z_{\lambda}(X)]\right),
\end{equation}
with equality holds if and only if $z_{\lambda}(X)=c$ almost surely, where $c$ is some constant. Suppose the equality of Equation~(\ref{eq:jensen}) holds, then\begin{align*}
    -\lambda \mathbb{D}_{\texttt{F}}(P_{1}||P_{\infty})&=\mathbb{E}_{\infty}[z_{\lambda}(X)]=c=\mathbb{E}_{1}[z_{\lambda}(X)]=\lambda \mathbb{D}_{\texttt{F}}(P_{\infty}||P_1).
\end{align*} 
It follows that $0\leq \mathbb{D}_{\texttt{F}}(P_{\infty}||P_{1})=-\mathbb{D}_{\texttt{F}}(P_{1}||P_{\infty})\leq 0$, which implies that $P_{\infty}=P_1$ almost everywhere. This leads to a contradiction to the assumption $P_{\infty}\neq P_1$. Thus, the inequality of Equation~(\ref{eq:jensen}) is \textit{strict}, and therefore $\delta<\lambda\mathbb{D}_{\texttt{F}}(P_{\infty}||P_{1})$. Hence, $\tilde{T}_{\texttt{SCUSUM}}$ is not trivial.

Define a sequence of stopping times: 
\begin{align*}
    &\eta_0 = 0,\\
    &\eta_1 = \inf \left\{t:\sum_{i=1}^t \tilde{z}_{\lambda}(X_i)<0\right\},\\
    &\eta_{k+1} = \inf \left\{t>\eta_k:\sum_{i=\eta_k+1}^t \tilde{z}_{\lambda}(X_i)<0\right\}, \; \text{for}\;  k\geq 1.
\end{align*}
By previous discussion, $\{\tilde{G}_n\}_{n\geq 1}$ is a nonnegative martingale under $P_{\infty}$ with mean 1. Then, for any $k$ and on $\{\eta_k<\infty\}$,
\begin{equation}
\label{eq:doobs}
P_{\infty}\left(\sum_{i=\eta_k+1}^n\tilde{z}_{\lambda}(X_i)\geq \tau \;\text{for some}\;  n>\eta_k \mid \mathcal{F}_{\eta_k} \right) \leq e^{-\tau},
\end{equation}
by Doob's submartingale inequality~\citep{doob1953stochastic}. Let
\begin{equation}
\label{eq:defm}
    M \de \inf \biggl\{k\geq 0: \eta_k<\infty \;\text{and} \; \sum_{i=\eta_k+1}^n\tilde{z}_{\lambda}(X_i)\geq \tau \; \text{for some}\; n>\eta_k\biggr\}.
\end{equation}
Combining Inequality~(\ref{eq:doobs}) and Definition~(\ref{eq:defm}),
\begin{align}
\label{eq:eq2}
    P_{\infty}(M\geq k+1\mid\mathcal{F}_{\eta_k})= 1-P_{\infty}\left(\sum_{i=\eta_k+1}^n\tilde{z}(X_i)\geq \tau  \;\text{for some} \; n>\eta_k\mid \mathcal{F}_{\eta_k}\right)\geq 1-e^{-\tau},
\end{align}
and
\begin{equation}
\label{eq:eq1}
    P_{\infty}(M> k)= \mathbb{E}_{\infty} [P_{\infty}(M\geq k+1\mid\mathcal{F}_{\eta_k})\mathbb{I}_{\{M\geq k\}}]=\mathbb{E}_{\infty}[P_{\infty}(M\geq k+1\mid\mathcal{F}_{\eta_k})]P_{\infty}(M> k-1).
\end{equation}
Combining Equations~(\ref{eq:eq1}) and (\ref{eq:eq2}), 
\begin{align*}
    \mathbb{E}_{\infty}[M] = \sum_{k=0}^{\infty}P_{\infty}(M> k)\geq \sum_{k=0}^{\infty}(1-e^{-\tau})^{k}= e^{\tau}.
\end{align*}

Observe that
\begin{align*}
    \tilde{T}_{\texttt{SCUSUM}}&=\inf \biggl\{n\geq 1:\sum_{i=\eta_k+1}^n\tilde{z}_{\lambda}(X_i)\geq \tau \; \text{for some}\; \eta_k<n \biggr\}\geq M,
\end{align*}
and $\tilde{T}_{\texttt{SCUSUM}}\leq T_{\texttt{SCUSUM}}$. We conclude that
$\mathbb{E}_{\infty}[T_{\texttt{SCUSUM}}]\geq \mathbb{E}_{\infty}[\tilde{T}_{\texttt{SCUSUM}}]\geq \mathbb{E}_{\infty}[M]\geq e^{\tau}$.
\end{enumerate}
\end{proof}

$\mathbb{E}_{\infty}[T_{\texttt{SCUSUM}}]$ is also referred to as the \textit{Average Run Length} (ARL)~\cite{page1955test}. Theorem~\ref{thm:arl} implies that the ARL increases at least exponentially as the stopping threshold increases. The following theorem gives the asymptotic performance of SCUSUM in terms of the detection delay under the control of the ARL.
\begin{theorem}
\label{thm:cond_edd}
   Subject to $\mathbb{E}_{\infty}[T_{\texttt{SCUSUM}}]\geq \gamma>0$, the stopping rule $T_{\texttt{SCUSUM}}$ satisfies
    \begin{equation} 
   \mathcal{L}_{\texttt{WADD}}(T_{\texttt{SCUSUM}}) \sim \mathcal{L}_{\texttt{CADD}}(T_{\texttt{SCUSUM}}) \sim \mathbb{E}_1[T_{\texttt{SCUSUM}}]\sim \frac{\log \gamma}{\lambda \mathbb{D}_{\texttt{F}}(P_{1}\|P_{\infty})},
    \end{equation}
as $\gamma \to \infty$.  
\end{theorem}
We first introduce a technical definition in order to apply~\cite[Corollary 2.2.]{woodroofe1982nonlinear} to the proof of Theorem 5.
\begin{definition}
A distribution $P$ on the Borel sets of $(-\infty, \infty)$ is said to be \textit{arithmetic} if and only if it concentrates on a set of points of the form $\pm nd$, where $d>0$ and $n=1, 2, \ldots$.
\end{definition}
\begin{remark}
    Any probability measure that is absolutely continuous with respect to the Lebesgue measure is non-arithmetic.
\end{remark}
\begin{proof}
Consider the random walk that is defined by 
\begin{equation*}
    Z^{\prime}(n) = \sum_{i=1}^nz_{\lambda}(X_i), \; \text{for}\; n\geq 1.
\end{equation*}
We examine another stopping time that is given by
\begin{equation*}
     T_{\texttt{SCUSUM}}^{\prime} \de \inf \{n\geq 1: Z^{\prime}(n) \geq \tau\}.
\end{equation*}
Next, for any $\tau$, define $R_{\tau}$ on $\{T_{\texttt{SCUSUM}}^{\prime} <\infty\}$ by 
\begin{equation*}
    R_{\tau} \de Z^{\prime}(T_{\texttt{SCUSUM}}^{\prime}) -\tau.
\end{equation*}
$R_{\tau}$ is the excess of the random walk over a stopping threshold $\tau$ at the stopping time $T_{\texttt{SCUSUM}}^{\prime}$.
Suppose the change point $\nu =1$, then $X_1, X_2,\ldots, $ are i.i.d. following the distribution $P_1$. Let $\mu$ and $\sigma^2$ respectively denote the mean $\mathbb{E}_{1}[z_{\lambda}(X)]$ and the variance $\text{Var}_1[z_{\lambda}(X)]$. Note that 
\begin{equation*}
    \mu =\mathbb{E}_{1}[z_{\lambda}(X)]= \lambda\mathbb{D}_{\texttt{F}}(P_1\|P_{\infty})>0,
\end{equation*}
and \begin{equation*}
    \sigma^2 = \text{Var}_1[z_{\lambda}(X)] = \mathbb{E}_1[z_{\lambda}(X)^2]-\left(\lambda\mathbb{D}_{\texttt{F}}(P_1\|P_{\infty})\right)^2.
\end{equation*}
Under the mild regularity conditions
given by \citet{hyvarinen2005estimation},
\begin{align*}
&\mathbb{E}_{1}[\mathcal{S}_{\texttt{H}}(X, P_{\infty})]^2 < \infty,\;\text{and} \\
&\mathbb{E}_{1}[\mathcal{S}_{\texttt{H}}(X, P_{1})]^2 < \infty.
\end{align*}
It implies that $\mathbb{E}_1[z_{\lambda}(X)^2]<\infty$ if $\lambda$ is chosen appropriately, e.g. $\lambda$ satisfy Inequality~(\ref{eq: condition}) and $\lambda$ is not arbitrary large. 
Therefore, by \citet[Theorem 1]{lorden1970excess},
\begin{equation*}
    \sup_{\tau \geq 0}\mathbb{E}_1[R_{\tau}]\leq \frac{\mathbb{E}_1[(z_{\lambda}(X)^{+})^2]}{\mathbb{E}_1[z_{\lambda}(X)]}\leq \frac{\mu^2+\sigma^2}{\mu},
\end{equation*}
where $z_{\lambda}(X)^{+} = \max (z_{\lambda}(X), 0)$.
Additionally,  $P_1$ must be non-arithmetic in order to have Hyv\"arinen scores well-defined. Hence, by \citet[Corollary 2.2.]{woodroofe1982nonlinear},
\begin{equation*}
    \mathbb{E}_{1}[T^{\prime}_{\texttt{SCUSUM}}]=\frac{\tau}{\mu}+\frac{\mathbb{E}_1[{R_{\tau}}]}{\mu}\leq \frac{\tau}{\mu}+\frac{\mu^2+\sigma^2}{\mu^2},\;\forall \tau \geq 0.
\end{equation*}
Observe that for any $n$, $Z^{\prime}(n)\leq Z(n)$, and therefore $T_{\texttt{SCUSUM}} \leq T_{\texttt{SCUSUM}}^{\prime}$. Thus, 
\begin{equation}
\label{eq:cadd_result}
    \mathbb{E}_{1}[T_{\texttt{SCUSUM}}]\leq \mathbb{E}_{1}[T_{\texttt{SCUSUM}}^{\prime}]\leq \frac{\tau}{\mu}+\frac{\mu^2+\sigma^2}{\mu^2},\;\forall \tau \geq 0.
\end{equation}
By Theorem 4, we select $\tau = \log \gamma $ to satisfy the constraint $\mathbb{E}_{\infty}[T_{\texttt{SCUSUM}}]\geq\gamma>0$. Plugging it back to Equation~(\ref{eq:cadd_result}), we conclude that, as $\gamma \to \infty$,
\begin{equation*}
    \mathbb{E}_{1}[T_{\texttt{SCUSUM}}] \sim \frac{\log \gamma}{\mu}=\frac{\log \gamma}{\lambda\mathbb{D}_{\texttt{F}}(P_{1}||P_{\infty})},
\end{equation*}
to complete the proof.

Due to the stopping scheme of SCUSUM, the expected time $\mathbb{E}_{\nu}[T_{\texttt{SCUSUM}}-\nu|T_{\texttt{SCUSUM}}\geq \nu]$ is independent of the change point $\nu$ (This is obvious, and the same property for CUSUM has been shown by~\citet{xie2021sequential}). Let $\nu = 1$, and we have \begin{equation*}
    \mathcal{L}_{\texttt{CADD}}(T_{\texttt{SCUSUM}}) = \mathbb{E}_{1}[T_{\texttt{SCUSUM}}]-1.
\end{equation*}
Thus, we conclude that 
\begin{equation*}
    \mathcal{L}_{\texttt{CADD}}(T_{\texttt{SCUSUM}})\sim \frac{\log \gamma}{\lambda \mathbb{D}_{\texttt{F}}(P_{1}\|P_{\infty})}.
\end{equation*}
Similar arguments applies for $\mathcal{L}_{\texttt{WADD}}(T_{\texttt{SCUSUM}})$.
\end{proof}
The value $\mathbb{E}_{1}[T_{\texttt{SCUSUM}}]$ is also referred to as the \textit{Expected Detection Delay} (EDD) in the literature. Theorems~\ref{thm:arl} and~\ref{thm:cond_edd} imply that the EDD increases linearly as the stop threshold $\tau$ increases subject to a constraint on ARL.
\begin{remark}
\label{remark: lambda}
It is worth noting that although results of our core results hold for a pre-selected $\lambda$ that satisfied the Inequality~(\ref{eq: condition}), the effect of choosing any other $\lambda^{\prime}$ amounts to the scaling of all the increments of SCUSUM by a constant factor of $\lambda^{\prime}/ \lambda$. This means that all of these results still hold adjusted for this scale factor. For instance, the result of Theorem \ref{thm:arl} can be modified to be written as $$\mathbb{E}_{\infty}[T_{\texttt{SCUSUM}}]\geq \exp \left\{\frac{\lambda  \tau}{\max(\lambda, \lambda^{\prime})}\right\},$$ for any $\lambda^{\prime} > 0$. It is easy to see that this scaling will change the statement of Theorem \ref{thm:cond_edd} accordingly to $$\mathbb{E}_{1}[T_{\texttt{SCUSUM}}]\sim \frac{\max(\lambda, \lambda^{\prime})}{\lambda }\frac{\log \gamma}{\lambda^{\prime}\mathbb{D}_{\texttt{F}}(P_{1}||P_{\infty})},$$ as $\gamma \to \infty$. In order to have the strongest results in Theorems~\ref{thm:arl} and \ref{thm:cond_edd}, we must choose $\lambda$ as close to $\lambda^*$ as possible.
\end{remark}

In the end, we consider a special case where pre- and post-change distributions are both multivariate Normal distributions. In this case, SCUSUM attains the asymptotic optimality in the sense of Pollak's and Lorden's metrics.
\begin{proposition}[Multivariate Normal Pre- and Post-change Distributions]
\label{prop: gaussian}
Assume that $X_1, \cdots, X_{\nu-1} \sim N(\boldsymbol{\theta}_0, \Sigma)$, and $X_{\nu}, X_{\nu+1}, \cdots \sim N(\boldsymbol{\theta}_1, \Sigma)$. Suppose $\boldsymbol{\theta}_0, \boldsymbol{\theta}_1\in \Theta\subset\mathbb{R}^{d}$ and $\Sigma\in \mathbb{R}^{d\times d}$ are known, and $\Sigma=\sigma_{\text{c}}\mathbf{I}_d$ where the scalar $\sigma_{\text{c}}>0$. Then the stopping rule $T_{\texttt{SCUSUM}}$ achieves the asymptotic optimality of Problem~(\ref{eq: wadd}) and Problem~(\ref{eq:pollak}) when $\gamma\to \infty$, namely, as $\gamma \to \infty$,
\begin{align}
    \mathcal{L}_{\texttt{WADD}}(T_{\texttt{SCUSUM}}) \sim\mathcal{L}_{\texttt{CADD}}(T_{\texttt{SCUSUM}})&\sim \frac{\log \gamma}{\mathbb{D}_{\texttt{KL}}(N(\boldsymbol{\theta}_1, \Sigma)\|N(\boldsymbol{\theta}_0, \Sigma))},
\end{align}
under the constraint that $\mathbb{E}_{\infty}[T_{\texttt{SCUSUM}}]\geq \gamma>0$.
\end{proposition}
\begin{proof}
By direct computation, it can see that
\begin{equation*}
    z_{\lambda}(X) = \lambda \left(
    -\frac{1}{2}(X-\boldsymbol{\theta}_0)^T\Sigma^{-2}(X-\boldsymbol{\theta}_0)+\frac{1}{2}(X-\boldsymbol{\theta}_1)^T\Sigma^{-2}(X-\boldsymbol{\theta}_1)
    \right),
\end{equation*}
where $\Sigma^{-2}$ is a short notation for $\Sigma^{-1}\cdot\Sigma^{-1}$. Then 
\begin{multline*}
    \mathbb{E}_{\infty}[\exp (z_{\lambda}(X))]
    =\int_{X\in\mathcal{X}}\frac{1}{\sqrt{2\pi}\operatorname{det}(\Sigma)}\exp \left(-\frac{1}{2}(X-\boldsymbol{\theta}_0)^T\Sigma^{-1}(X-\boldsymbol{\theta}_0)+\frac{\lambda}{2}(X-\boldsymbol{\theta}_0)^T\Sigma^{-2}(X-\boldsymbol{\theta}_0)\right.\\
    \left.-\frac{\lambda}{2}(X-\boldsymbol{\theta}_1)^T\Sigma^{-2}(X-\boldsymbol{\theta}_1)
    \right)dX.
\end{multline*}
The above integral can be calculated to be
\begin{equation*}
    \mathbb{E}_{\infty}[\exp (z_{\lambda}(X))]
    = \exp( - \lambda^2 (\boldsymbol{\theta}_0 - \boldsymbol{\theta}_1)^T \Sigma^{-3} (\boldsymbol{\theta}_0 - \boldsymbol{\theta}_1) + \lambda  (\boldsymbol{\theta}_0 - \boldsymbol{\theta}_1)^T \Sigma^{-2} (\boldsymbol{\theta}_0 - \boldsymbol{\theta}_1) ) .
\end{equation*}
Clearly, $\mathbb{E}_{\infty}[\exp (z_{\lambda}(X))] = 1$ if 
$$\lambda =   \frac{ (\boldsymbol{\theta}_0 - \boldsymbol{\theta}_1)^T \Sigma^{-2} (\boldsymbol{\theta}_0 -  \boldsymbol{\theta}_1) }    {(\boldsymbol{\theta}_0 - \boldsymbol{\theta}_1)^T \Sigma^{-3} (\boldsymbol{\theta}_0 -  \boldsymbol{\theta}_1)}.$$
The Fisher divergence and KL divergence between two Normal distributions can be calculated by
$$\mathbb{D}_{\texttt{F}}(\mathcal{N}(\boldsymbol{\theta}_1, \Sigma)||\mathcal{N}(\boldsymbol{\theta}_0, \Sigma)) = (\boldsymbol{\theta}_0 - \boldsymbol{\theta}_1)^T \Sigma^{-2} (\boldsymbol{\theta}_0 -  \boldsymbol{\theta}_1),$$
and
$$\mathbb{D}_{\texttt{KL}}(\mathcal{N}(\boldsymbol{\theta}_1, \Sigma)||\mathcal{N}(\boldsymbol{\theta}_0, \Sigma)) = (\boldsymbol{\theta}_0 - \boldsymbol{\theta}_1)^T \Sigma^{-1} (\boldsymbol{\theta}_0 -  \boldsymbol{\theta}_1),$$ respectively.
Thus \begin{equation*}
    \frac{ \lambda \mathbb{D}_{\texttt{F}}(P_{1}||P_{\infty}) }
{\mathbb{D}_{\texttt{KL}}(P_{1}||P_{\infty}) } = \frac{ [(\boldsymbol{\theta}_0 - \boldsymbol{\theta}_1)^T \Sigma^{-2} (\boldsymbol{\theta}_0 -  \boldsymbol{\theta}_1)]^2 } {[(\boldsymbol{\theta}_0 - \boldsymbol{\theta}_1)^T \Sigma^{-3} (\boldsymbol{\theta}_0 -  \boldsymbol{\theta}_1) ] [(\boldsymbol{\theta}_0 - \boldsymbol{\theta}_1)^T \Sigma^{-1} (\boldsymbol{\theta}_0 -  \boldsymbol{\theta}_1)]}.
\end{equation*}
Let $\{v_1, v_2, \cdots, v_d\}$ denote an orthonormal basis of eigenvectors of $\Sigma$, corresponding to its eigenvalues $\{\sigma_1, \sigma_2, \cdots, \sigma_d\}$. We can write $(\boldsymbol{\theta}_0 - \boldsymbol{\theta}_1)$ in this orthonormal basis as
$$(\boldsymbol{\theta}_0 - \boldsymbol{\theta}_1) = \sum_{k=1}^{d} c_k v_k.$$
Then, it follows from direct calculations that
$$\frac{ \lambda \mathbb{D}_{\texttt{F}}(P_{1}||P_{\infty}) }
{\mathbb{D}_{\texttt{KL}}(P_{1}||P_{\infty}) } = \frac{ ( \sum_{k=1}^{d} \frac{c_k^2}{\sigma_k^2})^2  } { ( \sum_{k=1}^{d} \frac{c_k^2}{\sigma_k^3})  ( \sum_{k=1}^{d} \frac{c_k^2}{\sigma_k} )}.$$
Applying the Cauchy-Schwarz inequality, we have
$$\lambda \mathbb{D}_{\texttt{F}}(P_{1}||P_{\infty}) \leq \mathbb{D}_{\texttt{KL}}(P_{1}||P_{\infty}),$$ 
with equality if and only if all the eigenvalues $\sigma_i, \, i=1, 2, \cdots, d$ for $c_i \ne 0$ are equal. In particular, in the case when $\Sigma$ is a scalar matrix, $\lambda \mathbb{D}_{\texttt{F}}(P_{1}||P_{\infty}) =  \mathbb{D}_{\texttt{KL}}(P_{1}||P_{\infty})$, and thus CUSUM and SCUSUM both achieve the same asymptotic performance.
\end{proof}
\begin{remark}
    We note that in the above Gaussian case (where the densities are normalized), whenever $$\lambda \mathbb{D}_{\texttt{F}}(P_{1}||P_{\infty}) <\mathbb{D}_{\texttt{KL}}(P_{1}||P_{\infty}),$$ the performance of CUSUM is superior to that of SCUSUM. However, CUSUM is not readily applicable to unnormalized models. This is a small penalty that SCUSUM pays in order to unleash its computational advantages.
\end{remark}

\section{Numerical Results}
\label{sec: results}
\noindent In this section, we conduct extensive numerical experiments on synthetic data to compare the performance of our method with various change detection algorithms.

\subsection{Experimental Setup}
\paragraph{Dataset} We simulate synthetic data streams from multivariate Normal distribution (MVN), a subfamily~\citep{yu2016statistical} of the exponential family (EXP), and the Gauss-Bernoulli Restricted Boltzmann Machine (GB-RBM)~\citep{LeCun2006ATO}. For the exponential family, we use the Hamiltonian Monte Carlo (HMC) to generate samples from the unnormalized models. We compute the normalizing constant by numerical integration to perform CUSUM based on log-likelihood. It is worth noting that this calculation is intractable when the dimension of EXP becomes large. The samples of GB-RBM are drawn using Gibbs sampling with $1000$ iterations to ensure convergence. We do not provide the results of CUSUM for GB-RBM because the exact log-likelihood of GB-RBM is hard to compute. Further details of the distributions can be found in Subsection~\ref{subsec: distributions}.
\paragraph{Baseline} We evaluate the performance in terms of empirical ARL and empirical CADD, where ARL and CADD are given by $\mathbb{E}_{\infty}[T]$ and $\mathbb{E}_{\nu}[T-\nu|T\geq \nu]$, respectively. When there is no change, we expect a large value of empirical ARL; when a change occurs, we expect a small value of empirical CADD. All the results of empirical CADD and empirical ARL are reported in a log scale. In all experiments, we set the change point as $\nu=500$. To make sure the data stream is long enough for detection, we fixed the total length as $10000$. The values of ARL range from $500$ to $20000$. Their theoretical properties have been discussed in Section~\ref{sec: SCUSUM}.

We compare the performance of SCUSUM with three other methods: CUSUM~\citep{page1955test}, Scan B-statistic~\citep{li2019scan, li2015m}, and CALM-MMD~\citep{cobb2022sequential}. For CUSUM and SCUSUM, we follow Algorithm~\ref{algm:scusum}. For a fixed ARL, the stopping threshold is selected by $\tau = \log$(ARL) according to Equation~(\ref{eq:optimality_cusum}). The Scan B-statistic algorithm was motivated by the B-statistic~\citep{zaremba2013b}. It is defined by the kernelized maximum mean discrepancy (MMD) between sliding bootstrap blocks of the data stream. \citet{cobb2022sequential} proposed a dynamic threshold-selecting algorithm, named CALM, which is applicable to most two-sample tests-like change detection methods. The CALM-MMD algorithm is returned by applying the CALM procedure to the kernelized two-sample MMD statistic~\citep{gretton2012kernel}. We implement the Scan B-statistic and CALM-MMD algorithms with the code released by~\cite{cobb2022sequential}. Both of these are kernelized MMD-based methods where the Gaussian radial basis function (RBF) kernel is employed. Their stopping thresholds are selected by past observations empirically, which can lead to significant miscalibration in practice, as shown by~\cite{cobb2022sequential} and later in our numerical results.

Other than the evaluation of the trade-off between ARL and CADD, we also investigate the performance of change detection in cases of slight changes, meaning that the pre- and post-change distributions are very close to each other. The closeness is measured by the magnitude of parameter drifts. Here, we run experiments by fixing the pre-change distribution and constructing the post-change distribution by perturbing the parameters of the pre-change distribution. For different families of distributions, we consider different magnitudes of perturbations. We repeat each experiment for $100$ trials.

\subsection{Synthetic Dataset}
\label{subsec: distributions}
\paragraph{Multivariate Normal Distribution (MVN)} We consider the multivariate normal distribution. Let $\boldsymbol{\mu}$ and $\Sigma$ respectively denote the mean and the covariance matrix. The corresponding score function is calculated by
\begin{equation*}
    S_{\texttt{H}}(X, P) = \frac{1}{2}(X-\boldsymbol{\mu})^{T}\Sigma^{-2}(X-\boldsymbol{\mu})-\operatorname{tr}(\Sigma^{-1}),
\end{equation*}
where the operator $\operatorname{tr}(\cdot)$ takes the trace of matrix.

We consider the pre-change distribution with mean $\boldsymbol{\mu}=(0,0)^T$ and covariance matrix $\Sigma = \left(\begin{matrix}
    1, &0.5\\
    0.5, &1
\end{matrix}\right)$. For the post-change distribution, we first investigate the scenario of mean shifts by fixing the covariance matrix $\Sigma = \left(\begin{matrix}
    1, &0.5\\
    0.5, &1
\end{matrix}\right)$ and assigning post-change means $\boldsymbol{\mu} = (0,0)^T+\epsilon_{\mu}$, where $+$ here is element-wise plus and $\epsilon_{\mu}$ is the perturbations of $\boldsymbol{\mu}$. 
We take values of $\epsilon_{\mu}$ from $0$ to $0.5$ with step size $0.05$. Next, we consider the case of covariance shifts. In this scenario, we fix the post-change mean as $\boldsymbol{\mu} = (0,0)^T$ and assign post-change covariance by $\Sigma = \left(\begin{matrix}
    1, &0.5\\
    0.5, &1
\end{matrix}\right)\circ\exp(\epsilon_{\log(\sigma^2)})$, where $\circ$ denotes the element-wise product and $\epsilon_{\log(\sigma^2)}$ denotes the element-wise perturbations of the covariance matrix. To make the perturbed covariance matrix positive-definite, we perturb the log of each component of the covariance matrix. We take the value of $\epsilon_{\log(\sigma^2)}$ vary from $0.05$ to $0.5$ by a step size $0.05$.

\paragraph{Exponential Family (EXP)} As introduced in Subsection~\ref{subsec:issues_llr_cusum}, we consider the Exponential family with the associated PDF given by
\begin{align}
    p_{\tau}(X) =\frac{1}{Z_{\tau}} \exp\left\{-\tau\left(\sum_{i=1}^dx_i^4+\sum_{1\leq i\leq d, i\leq j\leq d}x_i^2x_j^2\right)\right\}.\nonumber
\end{align}
The associated Hyvarinen score function is calculated by
\begin{equation*}
    S_{\texttt{H}}(X, P_{\tau}) = \frac{1}{2}\sum_{i=1}^d \left(\frac{\partial}{\partial x_i}\log p_{\tau}(X)\right)^2+\sum_{i=1}^d\frac{\partial^2}{\partial x_i}\log p_{\tau}(X),
\end{equation*}
where 
\begin{align*}
    \frac{\partial}{\partial x_i}\log p_{\tau}(X) &= -\tau \left(4x_i^3+2\sum_{1\leq i\leq d, i\leq j\leq d}x_ix_j^2\right), \;\text{and}\\
    \frac{\partial^2}{\partial x_i}\log p_{\tau}(X)&=-\tau \left (12x_i^2+2\sum_{1\leq i\leq d, i\leq j\leq d}x_j^2\right).
\end{align*}
We consider the pre-change distribution with $\tau =1$ and post-change distribution with $\tau =1+\epsilon_{\tau}$, where $\epsilon_{\tau}$ denotes the perturbations of the scale parameter $\tau$. We take values of $\epsilon_{\tau}$ from $0.1$ to $2.0$ by a step size $0.1$.

\paragraph{Gauss-Bernoulli Restricted Boltzmann Machine (GB-RBM)} As introduced in Subsection~\ref{subsec:issues_llr_cusum}, we consider the GB-RBM mode with the PDF given by $p_{\theta}(X)= \sum_{h\in \{0,1\}^{d_h}}p_{\theta}(X, H) = \frac{1}{Z_{\theta}}\exp\{-F_{\theta}(X)\}$, where $F_{\theta}(X)$ is the free energy given by 
\begin{equation*}
    F_{\theta}(X) = \frac{1}{2}\sum_{i=1}^{d_x} (x_{i}-b_i)^{2}\nonumber
    -\sum_{j=1}^{d_h} \operatorname{Softplus}\left(\sum_{i=1}^{d_x} W_{i j}x_{i}+b_{j}\right).
\end{equation*}
We compute the corresponding Hyv\"arinen score in a closed form
\begin{equation*}
    S_{\texttt{H}}(X, P_{\theta})= \sum_{i=1}^{d_x}\left[\frac{1}{2}\left(x_{i}-b_{i}+\sum_{j=1}^{d_h} W_{ij} \phi_{j}\right)^2+\sum_{j=1}^{d_h} W_{i j}^{2} \phi_{j}\left(1-\phi_{j}\right)-1\right],
\end{equation*}
where $\phi_{j} \de \operatorname{Sigmoid}(\sum_{i=1}^{d_x} W_{i j}x_{i}+b_{j})$. The $\operatorname{Sigmoid}$ function is defined as $\operatorname{Sigmoid}(y) \de (1+\exp(-y))^{-1}$.

The pre-change distribution is with the parameters $\mathbf{W} = \mathbf{W}_0$, $\mathbf{b}=\mathbf{b}_0$, and $\mathbf{c}=\mathbf{c}_0$, where each component of $\mathbf{W}_0$, $\mathbf{b}_0$, and $\mathbf{c}_0$ is randomly drawn from the standard Normal distribution $\mathcal{N}(0,1)$. For the post-change distribution, we assign the parameters $\mathbf{W}=\mathbf{W}_0+\epsilon_{\mathbf{W}}$, $\mathbf{b}=\mathbf{b}_0$, and $\mathbf{c}=\mathbf{c}_0$. Here, we only consider the shift of weight matrix $\mathbf{W}$, denoted as $\epsilon_{\mathbf{W}}$. Each component of $\epsilon_{\mathbf{W}}$ is drawn from $\mathcal{N}(0, \sigma_{\epsilon}^2)$. We let $\sigma_{\epsilon}$ take values from $0.005$ to $0.1$ with step size $0.005$.

\subsection{Experimental Results}
\paragraph{Detection Score} 
We illustrate instantaneous detection scores at time steps in Figure~\ref{fig:scores_t}. We control ARL to be fixed as $2000$. The data streams are generated from bivariate Normal distributions (MVN-$\epsilon _{\mu}$) with a mean drift $\epsilon _{\mu}=0.3$ at time $t=500$.  We report the averaged detection scores, marked as solid lines, and standard errors, marked as shadow intervals. As presented in Figure~\ref{fig:scores_t}, at the change point, both CUSUM and SCUSUM react immediately after the change occurs. In contrast, the detection scores of Scan B-statistic and CALM-MMD swing between the range of values $0$ and~$1$. In this case, the two MMD-based methods fail in detection. In particular, the detection scores of CUSUM and SCUSUM monotonically increase after the change happens. However, the detection scores of Scan B-statistic and CALM-MMD maintain a stable level after the change happens. Therefore, the results demonstrate that Scan B-statistic and CALM-MMD may fail to reach the threshold even after a sufficient number of time steps. 
\begin{figure*}[htbp]
\centering
 \includegraphics[width=0.9\linewidth]{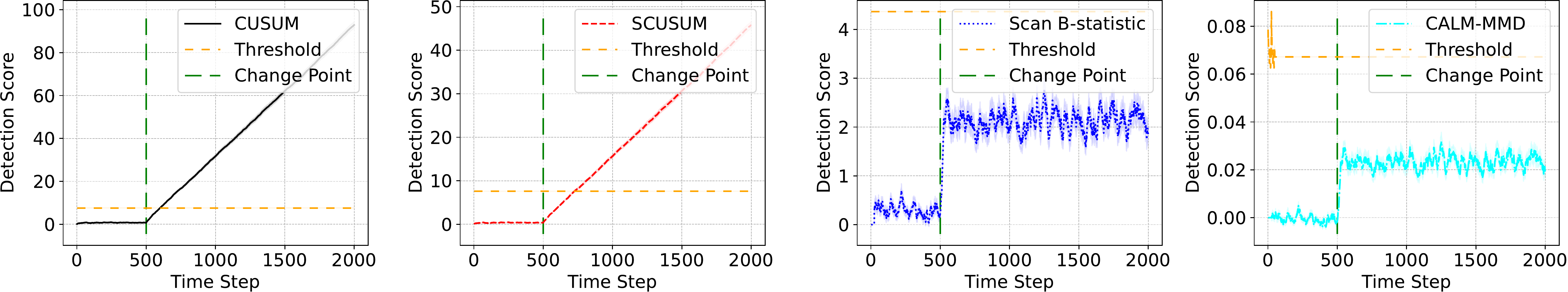}
   \caption{The results of Detection Score (before and after change) with MVN ($\epsilon_{\mu} = 0.3$) and ARL$=2000$.}
 \label{fig:scores_t}
\end{figure*}

\paragraph{Empirical CADD against ARL} In Figure~\ref{fig:cadd_arl}, we illustrate the empirical CADD against ARL in cases of bivariate Normal distribution mean drifts (MVN-$\epsilon _{\mu}$), bivariate Normal distribution covariance drifts (MVN-$\epsilon_{\log(\sigma^2)}$), scale parameter drifts of an exponential family (EXP-$\epsilon_{\tau}$), and weight matrix drifts of the GB-RBM (GB-RBM-$\sigma_{\epsilon}$), respectively. The notations $\epsilon _{\mu}$, $\epsilon_{\log(\sigma^2)}$, $\epsilon_{\tau}$, and $\sigma_{\epsilon}$ denote the magnitude of shits of the MVN mean, MVN covariance matrix, EXP scale parameter, and GB-RBM weight matrix, respectively. The results demonstrate that our proposed SCUSUM performs competitively with CUSUM in terms of empirical CADD against ARL. In particular, we see the red lines (SCUSUM) and the black lines (CUSUM) are in parallel, meaning that the empirical CADD of SCUSUM increases at a similar rate as that of CUSUM. Furthermore, SCUSUM can also outperform CUSUM for a fixed ARL in Figures~\ref{fig:cadd_arl}(b, c). 
\begin{figure}[htbp]
\centering
 \includegraphics[width=0.9\linewidth]{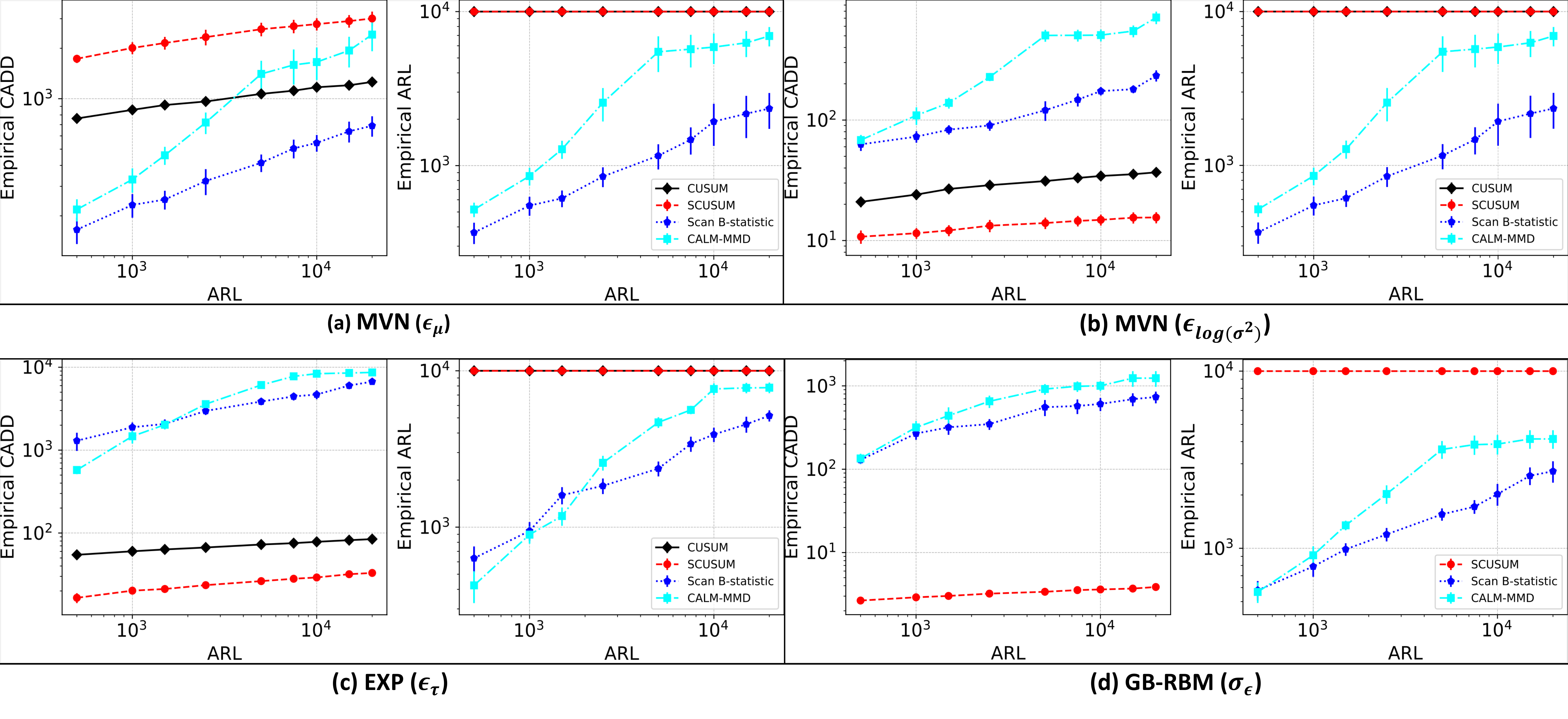}
   \caption{Empirical CADD against ARL and Empirical ARL against ARL for MVN ($\epsilon_{\mu}=0.1$), MVN ($\epsilon_{\log(\sigma^2)} = 0.5$), EXP ($\epsilon_{\tau} = 1.0$), and GB-RBM ($\sigma_{\epsilon} = 0.05$).}
 \label{fig:cadd_arl}
\end{figure}

The right columns of each subfigure in Figure~\ref{fig:cadd_arl} illustrate empirical ARL against ARL when no change happens throughout all time steps. The results demonstrate that CUSUM and SCUSUM can successfully control the false alarm rate, while MMD-based methods fail to do so. For the Normal distribution mean shifts, Scan B-statistic and CALM-MMD perform better than CUSUM and SCUSUM at low values of ARL. However, we point out that this gain is due to an out-of-control of false alarms, as illustrated in the right columns of Figure~\ref{fig:cadd_arl}(a). Furthermore, MMD-based methods not only fail to control false alarms but also perform worse than CUSUM and SCUSUM, as illustrated in Figures~\ref{fig:cadd_arl}(b-d). 

\paragraph{Empirical CADD against Changes}
We investigate the performance of the detection methods in cases of slight changes in Figure~\ref{fig:arl_perturb}, namely, pre- and post-change distributions are very close to each other. In the scenario of slight changes, CUSUM and SCUSUM perform better than MMD-based methods in Figures~\ref{fig:arl_perturb}(b-d). In particular, CUSUM and SCUSUM have much smaller empirical CADD when the magnitude of changes increases. Although MMD-based methods perform better than CUSUM and SCUSUM in Figure~\ref{fig:arl_perturb}(a), it is worth noting that it comes to the cost of out-of-control of false alarms as illustrated in Figure~\ref{fig:cadd_arl}.
\begin{figure}[htbp]
\centering
    \includegraphics[width=0.9\linewidth]{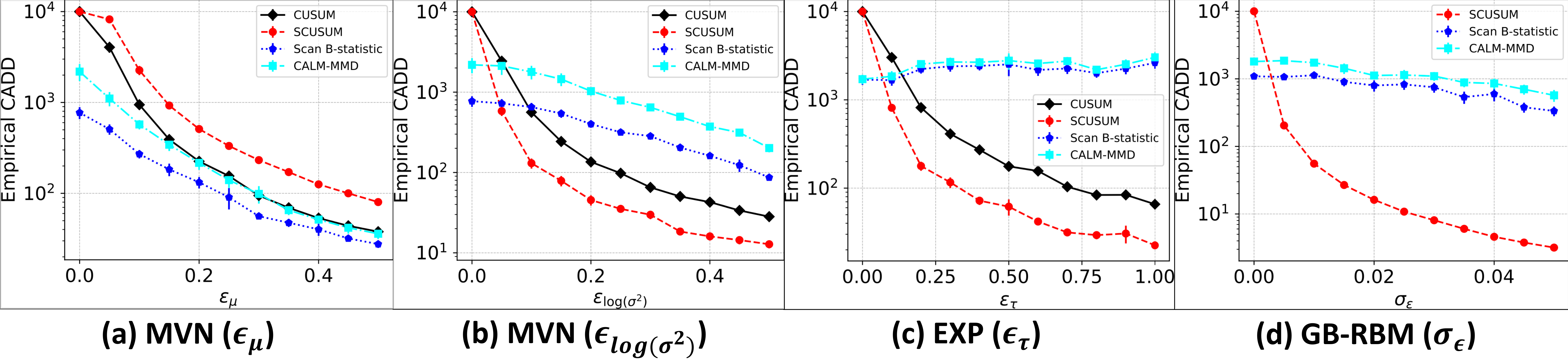}
    \caption{Empirical CADD against perturbations with ARL$=2000$}
    \label{fig:arl_perturb}
\end{figure}

\paragraph{Computation} We compare SCUSUM with other baselines in terms of computational costs by varying the dimensions of the EXP dataset. The computational cost is evaluated by the run time needed for the algorithm to stop detecting the change given one trial of the online data stream. Note that the normalizing constant for the exact likelihood is calculated offline given the knowledge of pre- and post-change distributions. This offline computation time is added to the run time of CUSUM. In Table~\ref{tab:cpu}, we demonstrate that when the dimension grows from $2D$ to $4D$, the run time needed for CUSUM grows significantly. It is due to the numerical integration of the exact log-likelihood calculation. Meanwhile, the run time of SCUSUM slightly grows due to the calculation of the Hyv\"arinen score. The run time of MMD-based methods stays consistent as the dimension grows. CALM-MMD requires a much longer run time due to its computation of selecting candidate thresholds. 
\begin{table}[ht]
\centering
\caption{Running times (in seconds) of change detection of each trial (pre- and post-change distributions belong to the exponential family)}
\vspace{0.1cm}
\label{tab:cpu}
\resizebox{0.5\columnwidth}{!}{
\begin{tabular}{@{}ccccc@{}}
\toprule
Detection Algorithms & $d=1$ & $d=2$ & $d=3$ & $d=4$ \\ \midrule
CUSUM & 2.4 & 2.9 & 294.8 & 66409.2 \\
SCUSUM & 2.2 & 9.1 & 21.0 & 38.4 \\
Scan B-statistic & 8.1 & 8.2 & 8.2 & 8.3 \\
CALM-MMD & 111.9 & 111.4 & 110.2 & 111.0 \\ \bottomrule
\end{tabular}
}
\end{table}

\paragraph{The Choice of $\lambda$}
In practice, we choose $\lambda$ as the positive root of the function $\tilde{h}(\lambda)$, as defined by Equation~(\ref{eq:empirical_conditon})in Section~\ref{sec:theoritical_analysis}. Different samples of past observations may determine different values of $\lambda$, which can cause the inconsistent performance of SCUSUM. We next investigate this problem through numerical simulations. In Figure~\ref{fig: lambda} (a) to (d), the data streams are generated from MVNs with $\epsilon_{\mu}=0.1$, MVNs with $\epsilon_{\log (\sigma^2)}=0.5$, EXPs with $\epsilon_{\tau}=1$, and GB-RBMs with $\sigma_{\epsilon}=0.05$. The first columns of Figure~\ref{fig: lambda} illustrate values of determined $\lambda$ varying from the size of past observations. The second (and the third) columns of Figure~\ref{fig: lambda} report the empirical CADD (respectively the empirical ARL) of SCUSUM varying from the size of past observations. We report all values in averages over $100$ random runs with error bars. 

As Figure~\ref{fig: lambda} demonstrates, as long as $m$ is large enough, the value of $\lambda$ is not too sensitive to different samples. In particular, when $m>100$, we see small standard errors in Figue~\ref{fig: lambda}(a)-(c). Accordingly, the performance of SCUSUM in terms of the empirical CADD tends to be stable. Note that in the case of GB-RBM (as shown by Figure~\ref{fig: lambda}(d)), we take $\lambda=1$ when $m<300$. It is because we can not numerically find the positive root of Equation~(\ref{eq:empirical_conditon}) given a small size of past observations. Finally, as shown in Column 3 of Figue~\ref{fig: lambda}, the empirical ARL is consistently under control.
\begin{figure}
    \centering
    \includegraphics[width=0.9\linewidth]{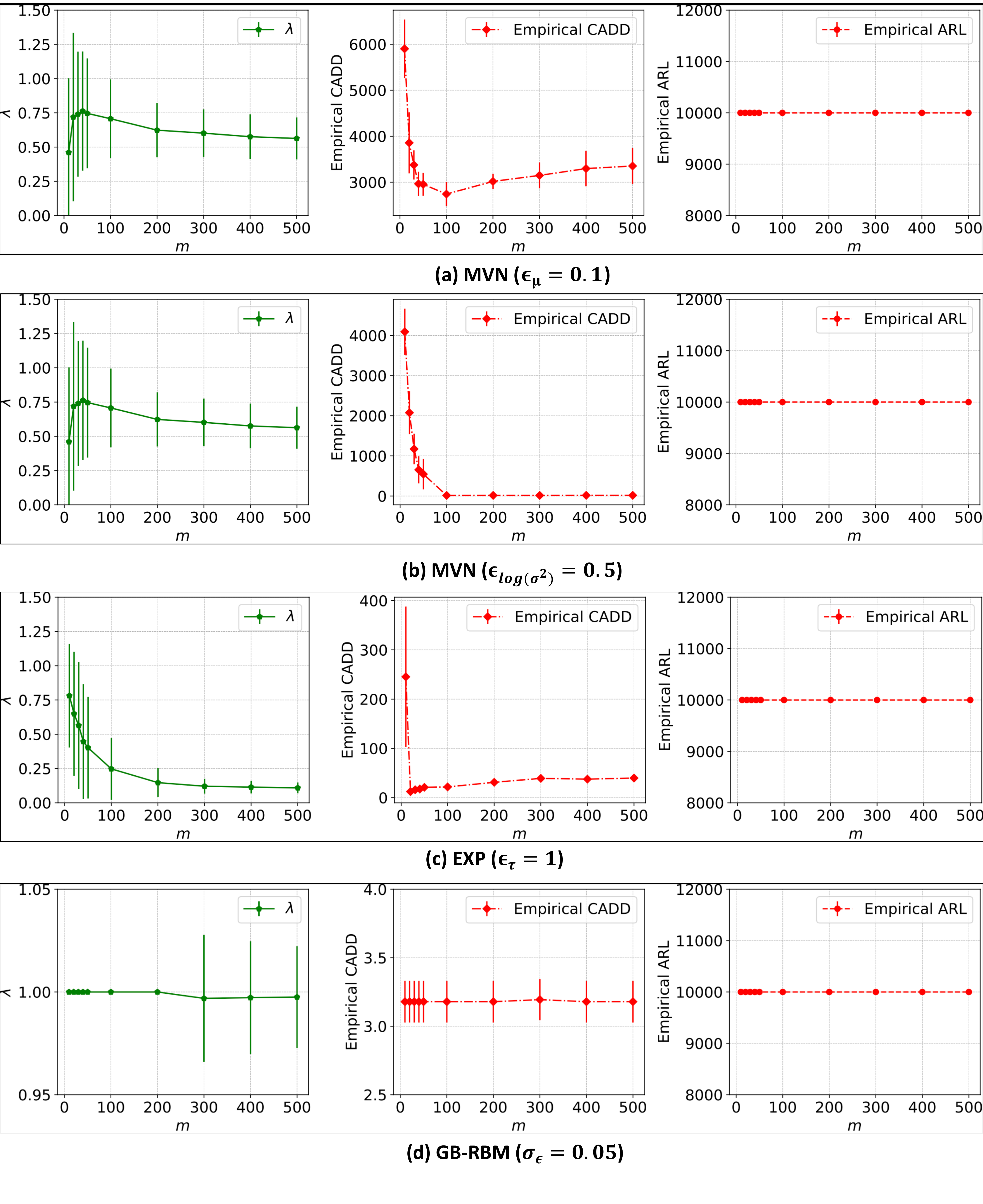}
    \caption{Column 1: $\lambda$ versus $m$; Column 2: Empirical CADD versus $m$; Column 3: Empirical ARL versus $m$.}
    \label{fig: lambda}
\end{figure}
\section{Concluding Remarks}
\label{sec: conclusion}
\noindent In this work, we proposed the SCUSUM algorithm to detect changes for unnormalized models. Our detection algorithm follows the classic CUSUM detection algorithm, sharing its computational advantage of recursive implementation. We analyzed the asymptotic properties of SCUSUM in the sense of Lorden's and Pollak's metrics. We also provided numerical results to demonstrate promising performance gains and reductions in computational complexity. In our future work, we will investigate the effect of relaxing the assumptions of knowing the post-change distribution and data independence. 


\appendices
\bibliographystyle{IEEEtranN}
\bibliography{IEEE TIT/reference}

\vfill

\end{document}